\documentclass{article}
\usepackage{amsthm}
\usepackage{amsmath,latexsym,amssymb}  
\usepackage{algorithm}
\usepackage{algorithmic}
\usepackage{graphicx}

\newtheorem{lemma}{Lemma}
\newtheorem{proposition}{Proposition}

\begin{document}

\title{\Large Optimally rotated coordinate systems for adaptive least-squares regression on sparse grids\thanks{Supported by CRC 1060 - The Mathematics of Emergent Effects
funded by the Deutsche Forschungsgemeinschaft.}}
\author{Bastian Bohn\thanks{Institute for Numerical Simulation, University of Bonn, Endenicher Allee 19b, 53115 Bonn, Germany,
\texttt{bohn}/\texttt{griebel}/\texttt{oettersh@ins.uni-bonn.de}} \and 
Michael Griebel\thanks{Fraunhofer Center for Machine Learning, Fraunhofer Institute for Algorithms and Scientific Computing SCAI, Schloss Birlinghoven, 53754 Sankt
Augustin, Germany} \footnotemark[2]
\and
Jens Oettershagen\footnotemark[2]}
\date{}

\maketitle

\newcommand{\ep}{\varepsilon}
\newcommand{\spara}{\shortparallel}
\newcommand{\sperp}{{\scriptscriptstyle\perp}}
\newcommand{\re}{{\mbox{\tiny ref}}}
\newcommand{\Hess}{{\mbox{Hess}}}

\renewcommand{\figurename}{Fig.}  

\newcommand{\SO}{{\mbox{\textbf{SO}}}}
\newcommand{\so}{{{\mathfrak{so}}}}
\newcommand{\St}{{\mbox{\textbf{St}}}}
\newcommand{\grad}{\nabla}
\newcommand{\rd}{\,\mathrm{d}}

\providecommand{\N}{\ensuremath{\mathbb{N}}}
\providecommand{\R}{\ensuremath{\mathbb{R}}}
\providecommand{\C}{\ensuremath{\mathbb{C}}}
\providecommand{\E}{\ensuremath{\mathbb{E}}}
\providecommand{\PP}{\ensuremath{\mathbb{P}}}

\providecommand{\cA}{\ensuremath{\mathcal{A}}}
\providecommand{\cB}{\ensuremath{\mathcal{B}}}
\providecommand{\cC}{\ensuremath{\mathcal{C}}}
\providecommand{\cD}{\ensuremath{\mathcal{D}}}
\providecommand{\cE}{\ensuremath{\mathcal{E}}}
\providecommand{\cF}{\ensuremath{\mathcal{F}}}
\providecommand{\cH}{\ensuremath{\mathcal{H}}}
\providecommand{\cI}{\ensuremath{\mathcal{I}}}
\providecommand{\cK}{\ensuremath{\mathcal{K}}}
\providecommand{\cL}{\ensuremath{\mathcal{L}}}
\providecommand{\cM}{\ensuremath{\mathcal{M}}}
\providecommand{\cN}{\ensuremath{\mathcal{N}}}
\providecommand{\cO}{\ensuremath{\mathcal{O}}}
\providecommand{\cP}{\ensuremath{\mathcal{P}}}
\providecommand{\cR}{\ensuremath{\mathcal{R}}}
\providecommand{\cS}{\ensuremath{\mathcal{S}}}
\providecommand{\cT}{\ensuremath{\mathcal{T}}}
\providecommand{\cU}{\ensuremath{\mathcal{U}}}
\providecommand{\cV}{\ensuremath{\mathcal{V}}}

\providecommand{\fM}{\ensuremath{\mathfrak{M}}}
\providecommand{\fp}{\ensuremath{\mathfrak{p}}}
\providecommand{\fP}{\ensuremath{\mathfrak{P}}}

\newcommand{\setu}{\mathrm{\mathbf{u}}} 
\newcommand{\setv}{\mathrm{\mathbf{v}}}	
\newcommand{\setw}{\mathrm{\mathbf{w}}}	
\newcommand{\setA}{\mathrm{\mathbf{A}}}
\newcommand{\setQ}{\mathrm{\mathbf{Q}}}
\newcommand{\setB}{\mathrm{\mathbf{B}}}
\newcommand{\setC}{\mathrm{\mathbf{C}}}

\newcommand{\bsa}{{\boldsymbol{a}}}	
\newcommand{\bsb}{{\boldsymbol{b}}}	
\newcommand{\bsc}{{\boldsymbol{c}}}	
\newcommand{\bsd}{{\boldsymbol{d}}}
\newcommand{\bsg}{{\boldsymbol{g}}}
\newcommand{\bsh}{{\boldsymbol{h}}}
\newcommand{\bsi}{{\boldsymbol{i}}}
\newcommand{\bsj}{{\boldsymbol{j}}}
\newcommand{\bsk}{{\boldsymbol{k}}}
\newcommand{\bsl}{{\boldsymbol{l}}}
\newcommand{\bsp}{{\boldsymbol{p}}}
\newcommand{\bsq}{{\boldsymbol{q}}}
\newcommand{\bsr}{{\boldsymbol{r}}}
\newcommand{\bss}{{\boldsymbol{s}}}
\newcommand{\bst}{{\boldsymbol{t}}}
\newcommand{\bsu}{{\boldsymbol{u}}}
\newcommand{\bsv}{{\boldsymbol{v}}}	
\newcommand{\bsw}{{\boldsymbol{w}}}	
\newcommand{\bsx}{{\boldsymbol{x}}}	
\newcommand{\bsy}{{\boldsymbol{y}}}	
\newcommand{\bsz}{{\boldsymbol{z}}}	

\newcommand{\bsxi}{{\boldsymbol{\xi}}}
\newcommand{\bsnu}{{\boldsymbol{\nu}}}

\newcommand{\bsgg}{{\boldsymbol{\gamma}}}	
\newcommand{\bsaa}{{\boldsymbol{\alpha}}}	
\newcommand{\bsbb}{{\boldsymbol{\beta}}}

\newcommand{\bsA}{{\boldsymbol{A}}}
\newcommand{\bsB}{{\boldsymbol{B}}}
\newcommand{\bsC}{{\boldsymbol{C}}}
\newcommand{\bsD}{{\boldsymbol{D}}}	
\newcommand{\bsH}{{\boldsymbol{H}}}	
\newcommand{\bsI}{{\boldsymbol{I}}}	
\newcommand{\bsL}{{\boldsymbol{L}}}	
\newcommand{\bsM}{{\boldsymbol{M}}}	
\newcommand{\bsQ}{{\boldsymbol{Q}}}	
\newcommand{\bsR}{{\boldsymbol{R}}}	
\newcommand{\bsS}{{\boldsymbol{S}}}	
\newcommand{\bsT}{{\boldsymbol{T}}}	
\newcommand{\bsU}{{\boldsymbol{U}}}	
\newcommand{\bsV}{{\boldsymbol{V}}}	
\newcommand{\bsW}{{\boldsymbol{W}}}	
\newcommand{\bsX}{{\boldsymbol{X}}}	
\newcommand{\bsY}{{\boldsymbol{Y}}}	
\newcommand{\bsZ}{{\boldsymbol{Z}}}	

\newcommand{\MM}{\mathfrak{M}}
\newcommand{\into}{\rightarrow}
\newcommand{\imp}{\Rightarrow}
\newcommand{\equ}{\Leftrightarrow}

\newcommand{\todo}[1]{\ \\ \fbox{\textbf{TODO: #1}} \\}
\newcommand{\sql}[1]{\ \\ \fbox{\parbox[c]{17cm}{ #1 }} \\}
\newcommand{\cusqrt}[1]{\sqrt[3]{#1}}

\newcommand{\divv}{\mbox{ div } }

\newcommand{\argmin}[1]{\underset{#1}{\operatorname{arg}\,\operatorname{min}}\;}
\newcommand{\argmax}[1]{\underset{#1}{\operatorname{arg}\,\operatorname{max}}\;}

\renewcommand{\d}{\ensuremath\text{d}}

\begin{abstract} \small\baselineskip=9pt
For low-dimensional data sets with a large amount of data points, standard kernel methods are usually not feasible for regression anymore. 
Besides simple linear models or involved heuristic deep learning models, grid-based discretizations of larger (kernel) model classes lead 
to algorithms, which naturally scale linearly in the amount of data points.
For moderate-dimensional or high-dimensional regression tasks, these grid-based discretizations suffer from the curse of dimensionality. 
Here, sparse grid methods have proven to circumvent this problem to a large extent.
In this context, space- and dimension-adaptive sparse grids, which can detect and exploit a given low effective dimensionality
of nominally high-dimensional data, are particularly successful. They nevertheless rely on an axis-aligned structure of the 
solution and exhibit issues for data with predominantly skewed and rotated coordinates.
\\
In this paper we propose a preprocessing approach for these adaptive sparse grid algorithms that determines an optimized,
problem-dependent coordinate system and, thus, reduces the effective dimensionality of a given data set in the ANOVA sense.
We provide numerical examples on synthetic data as well as real-world data to show how an adaptive sparse grid least squares algorithm benefits from our preprocessing method.
\end{abstract}

\section{Introduction}

In function regression, we determine $f$ from an admissible set $S$ which best approximates given data $(\bst_i, x_i)_{i=1}^N \subset \R^d \times \R$,  
i.e.\ $f(\bst_i) \approx x_i$ for $i=1,\ldots,N$.
While for deep neural network classes $S$ a complete theoretical foundation is still missing,
the famous representer theorem provides a direct way to compute $f$ if $S$ is a subset of a reproducing kernel Hilbert space ${\cH}$ \cite{Smola}.
However, the cost complexity of the underlying algorithm usually scales at least quadratically in $N$.

An easy and straightforward way to achieve linear cost complexity with respect to $N$ is to employ grid-based discretizations of localized functions from ${\cH}$. 
 However, standard tensor grids can only be used up to dimension $d=3$ because of the exponential dependence of the underlying computational costs with respect to $d$.
 This effect resembles the well-known \emph{curse of dimensionality}. Using a sparse grid discretization, which relies on the boundedness of mixed derivatives up to a fixed order,
 this exponential dependence is mitigated significantly while the discretization error is almost of the same order as for tensor grids \cite{Bungartz.Griebel:2004}.

 A further reduction of the costs in sparse grid regression
 can be achieved when adaptive variants are employed. 
 Here, the algorithm adapts in an a posteriori way to the underlying structure of the data \cite{pfluegerpeher}.
 This is particularly helpful if $f$ only depends on $k < d$ variables
 and is (nearly) constant along the remaining $d - k$ directions. In this 
 case $k$ is called the \emph{effective dimension} of $f$.
 Note that this is directly related to the concept of an analysis of variance 
 (ANOVA) decomposition in statistics. Furthermore, there is a direct connection between the ANOVA decomposition and certain sparse grid discretizations,
 see also \cite{Feuersaenger:2010}. 

Instead of using an Euclidean coordinate system, it is common 
to employ problem-dependent coordinates which simplify the description of the underlying problem:
Any (sufficiently differentiable) bijection $\psi: \R^d \rightarrow \R^d$ describes a coordinate transformation.
Then, if $f: \R^d \rightarrow \R$ approximates the data $(\bst_i, x_i)$, the function $\bar{f} = f \circ \psi: \R^d \rightarrow \R$
approximates the transformed dataset $(\psi^{-1}(\bst_i), x_i)_{i=1}^N$. 
The question is which bijections $\psi$ yield a small effective dimension of the transformed data set in the ANOVA sense, 
are still sufficiently cheap to represent, and allow for a more efficient approximation of $\bar{f}$ than the one using $\psi = \operatorname{id}$.

In this work, we concentrate on problems of effective (truncation) dimension $k < d$.
More specifically, for a problem with effective dimension $k$, we will focus on transformations from the \emph{Stiefel manifold} $V_{k}(\R^d)$,
which consists of all orthogonal $k$-frames in $\R^d$.
Our goal is to find a $\bsQ \in V_{k}(\R^d)$ such that the computational costs of learning the data $(\bsQ^T \bst_i, x_i)$ are as small as possible.
However, we do not want to solve the regression problem $(\bsQ^T \, \bst_i, x_i)_{i=1}^N$ for each candidate $\bsQ \in V_{k}(\R^d)$ involved in the optimization process.
Therefore, we make an approximation to the true function $f$ by a homogeneous $d$-variate polynomial $p \in \cP^{(d)}_m$ of total degree less than $m$,
which has relatively few degrees of freedom and is invariant under orthogonal transformations.
This means that we can determine $p$ before searching for the optimal $\bsQ \in V_{k}(\R^d)$
by minimizing the effective dimension of the low degree polynomial $p \circ \bsQ$ with respect to $\bsQ \in V_{k}(\R^d)$.
It is important to note that a crude approximation $p$ to $f$ is sufficient in our case since we are only interested in the rough behavior of the
lower-order ANOVA terms of $f \circ \bsQ$ and not in $f$ itself.
Overall, we aim to efficiently determine $\bsQ \in V_{k}(\R^d)$ such that $f \circ \bsQ$ can be well approximated by an adaptive sparse grid.

There are similarities to other established dimensionality reduction and data transformation algorithms.
For instance, a linear preprocessing technique to
solve multivariate integration problems has been used in \cite{Griebel.Holtz:2008, Imai/Tan:2004}.
Maximizing gradients of transformed functions is the main idea behind the active subspace method \cite{ConstantineBook}.
An active subspace approach based on polynomial surrogates for data-driven tasks can be found in \cite{ConstantineHokanson}. One of the main differences
to our proposed method is that the authors directly minimize the least-squares regression error of a linearly transformed polynomial. This leads to a coupled optimization problem 
in the polynomial $p$ and the linear transformation $\bsQ$. In our case, however, we exploit the fact that a sparse grid discretization
directly benefits from transformations $\bsQ$ for which 
$p \circ \bsQ$ has a small effective ANOVA dimension. Therefore, we will fix $p$ a priorily and then search for a $\bsQ$ which minimizes the effective dimension of $p \circ \bsQ$.
Subsequently, 
we will learn a sparse grid function to approximate $f \circ \bsQ$. In this way, we rely on the polynomial surrogate $p$ only to determine $\bsQ$.
This makes the optimization with respect to $\bsQ$ much easier and still allows for efficient sparse grid discretizations of the underlying regression problem. This procedure is also sketched in Figure \ref{fig:idea}.

The remainder of this article is organized as follows:
In Section 2, we introduce the ANOVA-decomposition and concepts of effective dimensionality.
In Section 3, we reduce the effective dimensionality using coordinate systems from $V_k(\R^d)$. 
In Section 4, we apply this method to machine learning by transforming a data set and learning 
an adaptive sparse grid approximation on the transformed set.
Section 5 contains numerical results to validate the benefit of the coordinate transformation.
In Section 6 we give some concluding remarks.

\begin{figure}[htb] \label{fig:idea}
 \includegraphics[width=0.32\linewidth]{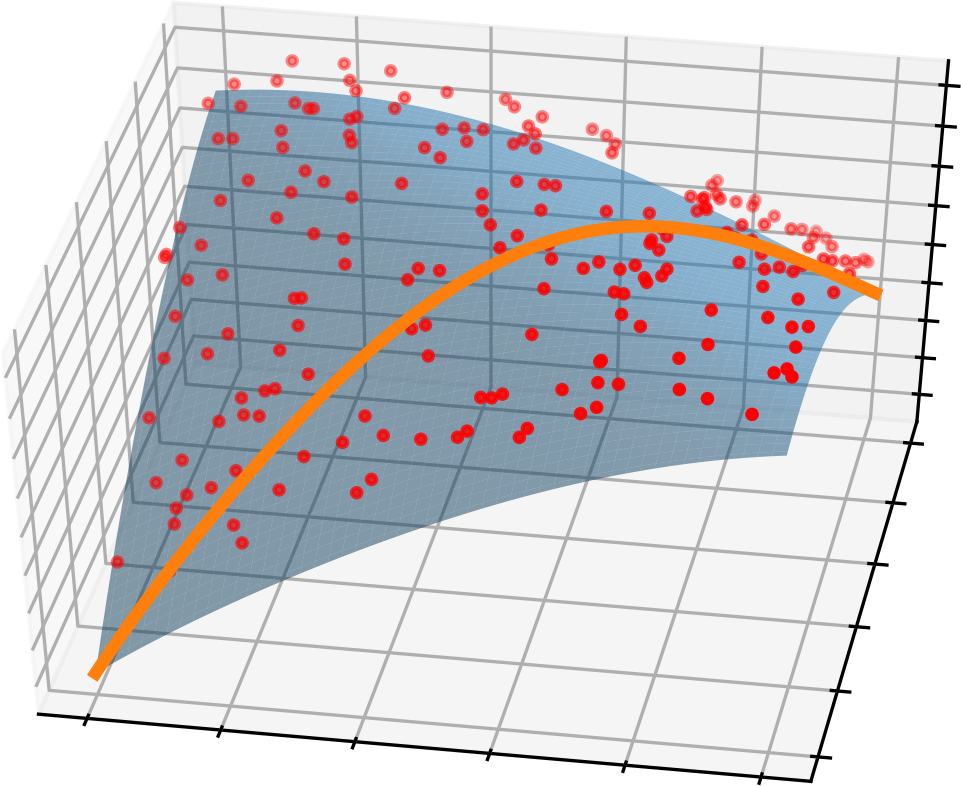}
\includegraphics[width=0.32\linewidth]{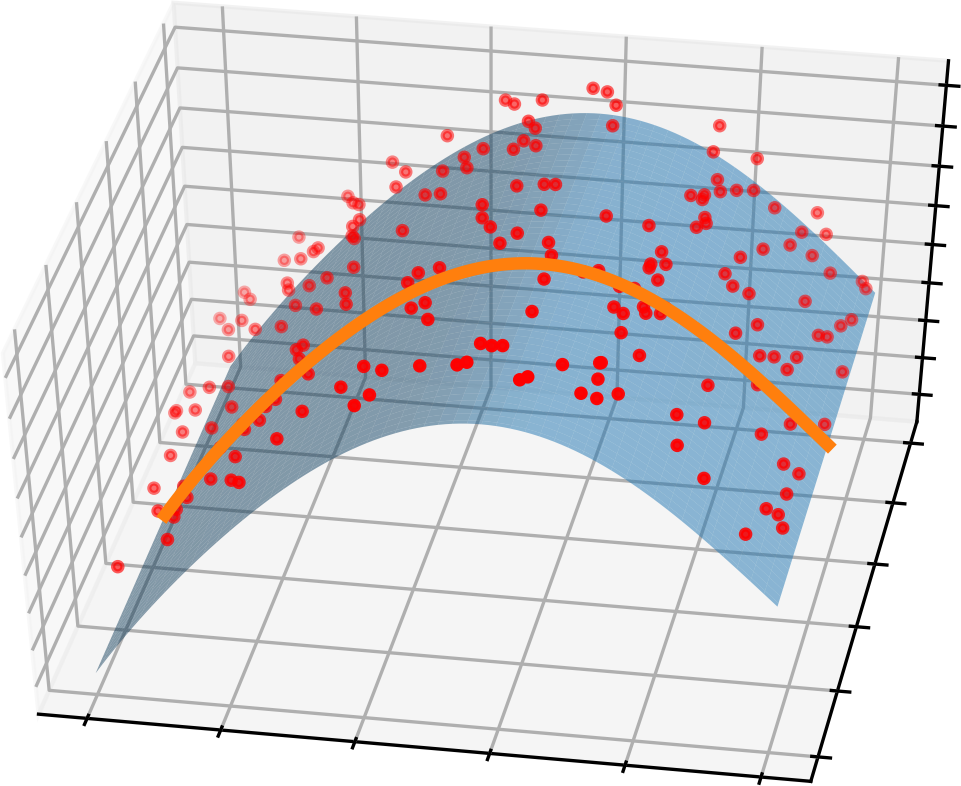}
\includegraphics[width=0.32\linewidth]{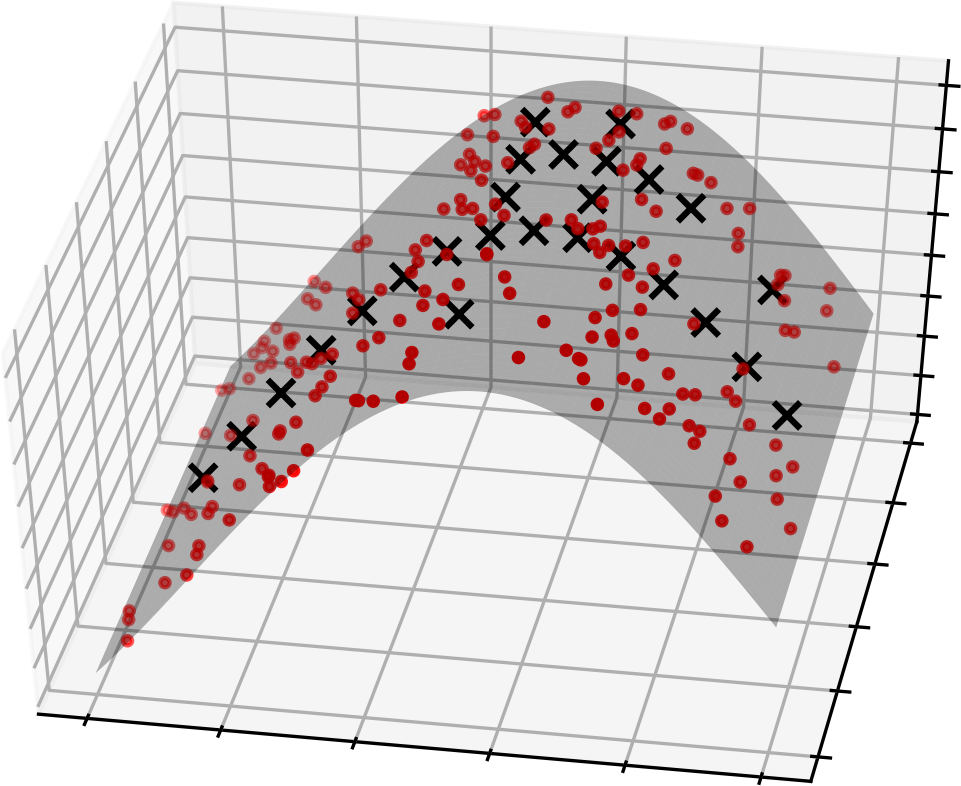}
\caption{Data approximated by a polynomial (left). Polynomial and data after transformation (mid). Adaptive sparse grid approximation of transformed data. (right)}
\end{figure}

\section{Effective dimensionality of functions}
In this section we recall the classical analysis-of-variance (ANOVA) decomposition and the concept of effective dimensionality.
To this end, let  
$\Omega \subseteq \R$ be a fixed domain. For all subsets $\setu \subseteq \cD := \{1,2,\ldots, d\}$,
we define the $|\setu|$-dimensional product domains $\Omega^{|\setu|} \subseteq \R^{|\setu|}$.
In the following, we write $\setu^c$ to denote $\cD \setminus \setu$. Let 
	$\rd \mu(\bsx) = \prod_{j=1}^d \rd \mu_j(x_j)$
be a $d$-dimensional product of probability measures $\mu_j$ on the Borel-algebra of $\Omega$.
The associated measures on $\Omega^{|\setu|}$ are given by $\rd\mu_\setu(\bsx_\setu) := \prod_{j\in \setu} \rd\mu_j(x_j)$,
where $\bsx_\setu$ denotes the $|\setu|$-dimensional vector which contains those components of $\bsx$ with indices in $\setu$.
Let $\cT^{(d)} := L_2(\Omega^d, \mu)$ 
be endowed with the inner product  
\[ (f,g)_{\mu} := \int_{\Omega^d} f(\bsx) g(\bsx) \, \rd\mu(\bsx) \]
and its induced norm $\|f\|^2_{2, \mu} 
= \int_{\Omega^d} f(\bsx)^2 \, \rd\mu(\bsx)$.
For $\setu \subset \cD$, the spaces $\cT^{\setu} := L_2(\Omega^{\setu}, \mu_\setu)$ will be treated as subspaces of $\cT^{(d)}$
by viewing their elements as $d$-variate functions that only depend on the variables $j \in \setu$, i.e.
$$\cT^\setu=\{ f \in \cT^{(d)}: f(\bsx_\setu, \bsy_{\setu^c}) = f(\bsx_\setu, \tilde{\bsy}_{\setu^c}) \ \forall \ {\bsy}_{\setu^c}, \tilde{\bsy}_{\setu^c} \}$$ 

The basic idea behind the ANOVA decomposition is 
to define projections $\cT^{(d)} \rightarrow \cT^\setu$ which will then be employed
to decompose a $d$-variate 
function $f \in \cT^{(d)}$ into a sum of low-dimensional functions, i.e.
\begin{align*}
	f(\bsx) & = f_\emptyset + \sum_{i=1}^d f_{\{i\}}(x_i) + \sum_{\substack{i,j=1,\ldots,d \\ i < j}} f_{\{i,j\}} (x_i, x_j) + \ldots + f_\cD(x_1,\ldots, x_d)
\end{align*}
This sum will be abbreviated by $f(\bsx) = \sum_{\setu \subseteq \cD} f_\setu(\bsx_\setu)$.

\subsection{The ANOVA decomposition.}
We begin by defining the orthogonal projectors $P_\setu: \cT^{(d)} \into \cT^{\setu}$ via
\begin{align*}
	P_\setu (f) (\bsx_\setu) & :=  \int_{\Omega^{\setu^c}} f(\bsx) \, \rd\mu_{\setu^c} (\bsx_{\setu^c}) && \mbox{for } \setu \subsetneq \cD \\
	P_\setu (f) (\bsx) & :=  f(\bsx)  && \mbox{for } \setu = \cD.
\end{align*}
The projections are orthogonal in $\cT^{(d)}$ and hence give the $\cT^{(d)}$-optimal low-dimensional approximation to $f \in \cT^{(d)}$ by functions from $\cT^\setu$.
Now, let 
\begin{equation} \label{eqn:anova_ii}
	f_\setu(\bsx_\setu) := P_\setu(f)(\bsx_\setu) - \sum_{\setv \subsetneq \setu} f_\setv(\bsx_\setv) 
\end{equation}
for all $\bsx_\setu \in \Omega^{\setu}$.
Then it holds 
\begin{equation*}
	f(\bsx) = \sum_{\setu \subseteq \cD} f_\setu(\bsx_\setu) ~ \mbox{ and } \left( f_\setu , f_\setv \right)_{\mu} = 0 \ \text{for } \setu \neq \setv.
\end{equation*}
The $2^d$ summands $f_\setu, \setu \subseteq \cD$, describe the dependence of $f$ on the subset of variables contained in $\setu$. 
Analogously to the definition of the variance of $f \in \cT^{(d)}$, the variance of the ANOVA term $f_\setu$ for a $\setu \neq \emptyset$ is given by
	\begin{align*}
		\sigma_{\setu, \mu} ^2(f) & = \int_{{\Omega}^{\setu}} f_\setu^2 \, \rd\mu_\setu - \underbrace{\left( \int_{{\Omega}^{\setu}} f_\setu \, \rd\mu_\setu \right)^2}_{=0}
		 =  \int_{{\Omega}^{\setu}} f_\setu^2 \, \rd\mu_\setu.
	\end{align*}
	
Then, due to the orthogonality of the ANOVA-decomposition, the variance of $f$ can be decomposed into the sum of the variances of all ANOVA terms, i.e.
	\begin{equation*} \label{eqn:varianz-zerlegung}
		\sigma^2_{\mu}(f)= \sum_{|\setu| > 0} \sigma^2_{\setu, \mu}(f).
	\end{equation*}
We define the auxiliary quantity
\begin{align*} \label{eqn:Du1}
	D_\setu(f) := & \ \sum_{\setv\subseteq \setu} \sigma^2_{\setv,\mu}(f) \stackrel{\eqref{eqn:anova_ii}}{=} 
	\int_{\Omega^\setu} P_\setu(f)^2 \, \rd \mu_{\setu} - f_\emptyset^2
	\\ 
	= & \ \int_{\Omega^\setu} \left(  \int_{\Omega^{\setu^c}} f \rd \mu_{\setu^c} \right)^2 \rd \mu_\setu - f_\emptyset^2 \nonumber
\end{align*}
and use it to determine $\sigma_{\setu,\mu}^2(f)$ recursively via
\begin{equation} \label{eq:varViaSobol}
	\sigma_{\setu,\mu}^2(f) = D_\setu(f) - \sum_{\setv\subsetneq \setu} \sigma^2_{\setv,\mu}(f).
\end{equation}
The value of $D_\setu$ is given by the following Lemma, which we will exploit later. This result has been 
proven in Theorem 2 of \cite{Sobol:2001}.
\begin{lemma} \label{lemma:Du}
For $\setu \subseteq \cD$ it holds
	\begin{align*} \label{eqn:Du_lemma}
		D_\setu(f) + f_\emptyset^2 & = \int_{\Omega^{2d-|\setu|}} f(\bsx_\setu, \bsx_{\setu^c})\, f(\bsx_\setu, \bsy_{\setu^c}) \ \rd\mu_\setu(\bsx_\setu) \rd\mu_{\setu^c}(\bsx_{\setu^c})
		\rd\mu_{\setu^c}(\bsy_{\setu^c}).
	\end{align*}
\end{lemma}

The next proposition shows that the $\sigma_{\setu,\mu}(f)$ are invariant with respect to component-wise coordinate transformations. 

\begin{proposition}
    Let $\Omega, \widehat{\Omega} \subseteq \R$ and let $\mu = \bigotimes_{j=1}^d \mu_j$ be a product measure on $\Omega^d$.
    Moreover, let $\Phi: \widehat{\Omega}^d \to \Omega^d$ be defined by $\Phi(\bsx) := (\phi_1(x_1),\ldots,\phi_d(x_d))$, 
    where each $\phi_j: \widehat{\Omega} \leftrightarrow \Omega$ is a diffeomorphism.    
    Then it holds
    \begin{equation}
    	\sigma^2_{\setu,\mu}(f) = \sigma^2_{\setu,\mu\circ\Phi} (f \circ \Phi) .
    \end{equation}
\end{proposition}
\begin{proof}
 This is a direct consequence of the change of variables formula.
\end{proof}

This result is of special interest if the component functions $\phi_j$ of the transformation are the inverses of the cumulative distribution functions of the $\mu_j$.

Finally, in order to compare the dependence of functions on their respective ANOVA terms, we define the so-called sensitivity coefficients
 \[
  s_\setu(f) := \frac{1}{\sigma_\mu^2(f)} \sigma^2_{\setu,\mu}(f) ~~ \mbox{ for all } \ \setu \subseteq \cD.
 \]
 These coefficients describe the relative importance of the coordinate directions in $\setu$.
 Note that $\sum_{\setu \subseteq \cD} s_\setu = 1$.

\subsection{Notions of effective dimensionality.}
The term \emph{effective dimensionality} is based on the insight that the ANOVA terms of higher cardinality contribute 
much less to the total variance than the lower-order terms for many application-driven problems and that methods for 
their solution can benefit from this property. 
In \cite{Owen:2003} the \emph{mean dimension} is defined by weighting the variances $\sigma^2_\setu$ of the single ANOVA terms
$f_\setu$ with their cardinalities $|\setu|$, i.e.\ higher-order terms get penalized stronger than lower-order terms.
Notions of effective dimensionality which are not based on the classical ANOVA-decomposition, but rather on the 
anchored ANOVA approach can be found in \cite{Griebel.Holtz:2008}. 

In this paper we will employ a {\em generalization} of the mean dimension. To this end, 
let $\nu_\setu > 0$ be an arbitrary set of weights for all $\setu \subseteq \cD$. We define
\begin{equation} \label{eqn_meandim_generalization}
 d_{{\boldsymbol{\bsnu}}} (f) := \sum_{\setu \subseteq \cD} \nu_\setu \,  s_\setu(f) .
\end{equation}
Thus, sensitivity coefficients in the ANOVA decomposition can be weighted differently, e.g.\ according to the number of variables present in each ANOVA term.

\section{Minimizing the effective dimensionality}
In the following, we aim to
find a coordinate transformation $\psi: \Omega^d \rightarrow \R^d$ to reduce the generalized mean dimension \eqref{eqn_meandim_generalization} 
for a given $f: \R^d \rightarrow \R$. This directly leads to the minimization problem
\begin{equation}  \label{eqn_mini_func}
	\mathfrak{M}_f(\psi) := \sum_{\setu \subseteq \cD} \nu_\setu \, s_\setu(f \circ \psi) 
	\longrightarrow \min_{\psi \in \Psi}!,
\end{equation}
where $\nu_\setu > 0$ are prescribed weights that should penalize higher-order terms and $\Psi$ is a class of suitable diffeomorphisms. The hope is that there exists a 
$\psi$ such that $f \circ \psi$ has a substantially smaller dependence on higher-order ANOVA terms than the original $f$ did.
At this point we should realize that the approximation of $\psi$ also involves certain costs. 
Therefore, we need the class $\Psi$ to be both powerful enough to reduce the effective dimension and small enough to rely only on a few degrees of freedom 
so that we do not just shift the costs from the approximation of the outer function $f \circ \psi$ to the inner function $\psi$.

\subsection{Equivalent maximization problem.}
Penalizing higher-order ANOVA terms makes the functional \eqref{eqn_mini_func} expensive or even impossible to evaluate as these terms contribute the most to its value 
and their evaluation is based on the evaluation of all lower-order terms for $\setv \subsetneq \setu$. Therefore,
we are looking for a reformulation of the minimization task \eqref{eqn_mini_func} which circumvents this problem. 
To this end, note that \begin{align*}
	\MM_f(\psi) & = \sum_{\setu \subseteq \cD} \nu_\setu \, s_\setu(f \circ \psi)  \\
		& = \sum_{|\setu|<d} \nu_\setu \, s_\setu(f\circ \psi) + \nu_\cD \,  s_\cD(f\circ \psi)  \\
		& =  \sum_{|\setu|<d} \left( \nu_\setu - \nu_\cD \right)  \,  s_\setu(f\circ \psi) + \nu_\cD .
\end{align*}
Therefore, a minimizer of $\MM_f(\psi)$ is also a maximizer of $- \frac{1}{\nu_\cD} \MM_f(\psi)$ and vice versa and
 we obtain the following equivalent maximization problem
\begin{equation} \label{eq:maxFunc}
 \hat{\MM}_{f}(\psi) := \sum_{\setu \subsetneq \cD} \left( 1 - \frac{\nu_\setu}{\nu_\cD} \right) \, s_\setu(f \circ \psi) \longrightarrow \max_{\psi \in \Psi}!.
\end{equation}
The main advantage of considering
\eqref{eq:maxFunc} instead of \eqref{eqn_mini_func} is that we can now omit sets $\setu$ with large $| \setu |$ in \eqref{eq:maxFunc} and 
focus the optimization task to sets with small $| \setu |$, i.e. lower-dimensional terms only.
%

\subsection{Choice of the weights.}
Let $1 \leq k < d$. In the remainder of the article, we will use
\begin{equation*} \label{eq:weights}
 \nu_{\setu} := \left\{ \begin{array}{cl} 1 - \exp\left(- \max \{ j \in \setu \} \right) & \text{ if } \setu \subseteq \{1,\ldots,k\}, \\ 1 & \text{ else.} \end{array} \right.
\end{equation*}
Now we only need to evaluate ANOVA terms of $f \circ \psi$ corresponding to subsets of the first $k$ variables since
$$
\hat{\MM}_{f}(\psi) = 
\sum_{\setu \subseteq \{1,\ldots,k\}} 
\exp\left(- \max \{ j \in \setu \} \right) s_{\setu}(f \circ \psi).
$$
In this sense, we try to find a $\psi$ such that ideally all (or at least most)
of the variance of $f \circ \psi$ resides in these terms. Such a function is said to have truncation dimension $k$ in the ANOVA sense.

For the subclass of orthogonal projections $\Psi$, which we will focus on in this paper, and for a measure $\mu$ which is invariant 
under orthogonal transformations, such as the Lebesgue or the Gaussian measure on $\R^d$, we obtain 
$\sigma^2_{\mu}(f) = \sigma^2_{\mu \circ \psi}(f \circ \psi)$ for all $\psi \in \Psi$. Therefore, we can simply omit $\sigma^2_{\mu}(f \circ \psi)$ 
in the maximization functional in this case and we just maximize
\begin{equation} \label{eq:finalMaxProb}
 \hat{\MM}_{f}(\psi) := 
\sum_{\setu \subseteq \{1,\ldots,k\}} \exp\left(- \max \{ j \in \setu \} \right) \sigma^2_{\setu, \mu}(f \circ \psi),
\end{equation}
which we can evaluate by using \eqref{eq:varViaSobol} and Lemma \ref{lemma:Du}.

\subsection{Orthogonal transformations.}

Due to our specific choice of weights, we can actually 
restrict ourselves to transformations $\phi: \Omega^k \subseteq \R^k \to \R^d$ instead of having to look for maps with a domain in $\R^d$.
Therefore, let 
$$
V_k(\R^d) := \left\{ \bsQ \in \R^{d \times k} \mid \bsQ^T \bsQ = \bsI \right\},
$$
be our class of valid transformations. Here, the rows of $\bsQ \in V_k(\R^d)$ represent an orthogonal $k$-frame in $\R^d$.
This class is actually a submanifold of $\R^{d \times k}$ and it is known as the so-called \emph{Stiefel manifold}.

As we see, maximizing $\hat{\MM}_{f}(\bsQ)$ over $\bsQ \in V_k(\R^d)$ is a highly nonlinear task with possibly nonunique maximizers. 
The existence of maximizers can be guaranteed for continuous functions $f$ since $\hat{\MM}_{f}$ is a continuous functional in that case and 
$V_k(\R^d) \subset \R^{d \times k}$ is compact.
As mentioned earlier, we will substitute $f$ by a polynomial surrogate for the actual optimization for which we can, therefore,
guarantee the existence of maximizers.

\subsection{Polynomials as invariant basis.} \label{sec:poly}
The largest part of the costs in evaluating $\hat{\MM}_f(\bsQ)$ is the evaluation of each $\sigma^2_{\setu,\mu}(f \circ \bsQ)$, 
which requires the approximation of high-dimensional integrals. Therefore, 
we discretize $f$ in a basis which allows to compute $\sigma^2_{\setu,\mu}(f \circ \bsQ)$ analytically for $\Omega = \R$ and which is closed under orthogonal transformations.
%
%
To this end, we employ a total degree polynomial space with a homogeneous basis in $\R^k$, i.e.\ we take the basis set
\begin{equation*} \label{eqn_polynomial_basis}
  \cB^{(k)}_m := 
  \{ \bsx^\bsaa = x_1^{\alpha_1} \ldots x_k^{\alpha_k} \mid |\bsaa|_1 \leq m \}
\end{equation*}
for some $m \in \N$.
Note that $K:= | \cB^{(k)}_m | = \binom{k+m}{k}$.



\begin{lemma} \label{lemma:drehinvarianz}
The basis $\cB^{(d)}_m$
spans the total degree polynomial space $\cP^{(d)}_m$ on $\R^d$ which is invariant with respect to all orthogonal transformations $\bsQ \in V_k(\R^d)$, i.e.
\[
	\phi \circ \bsQ \in  \mathrm{span}\{ \cB^{(k)}_m \} ~~ \forall \ \phi \in \cP^{(d)}_m, \bsQ \in V_k(\R^d).
\]
\begin{proof}
Let $\omega_{\bsaa}(\bsx) := \bsx^\bsaa$. Using the multinomial theorem with $\beta \in \N_0^k$, we obtain
	\begin{align*}
		\omega_{\bsaa} \circ \bsQ(\bsx)  &= \prod_{i=1}^d \left( \sum_{j=1}^k Q_{ij} x_j \right)^{\alpha_i} = \ \prod_{i=1}^d \underbrace{\sum_{|\bsbb|_1=\alpha_i} \frac{\alpha_i!}{\beta_1! \cdots \beta_k!} \prod_{j=1}^k \left( Q_{ij} x_j \right)^{\beta_j}}_{\in \mathrm{span}\{ \cB^{(k)}_{\alpha_i}\}}.
	\end{align*}
	Since $\mbox{deg}(P \cdot S) = \mbox{deg}(P) + \mbox{deg}(S)$ holds for polynomials $P$ and $S$ and because of $| \bsaa|_1 \leq m$, the claim follows.
\end{proof}
\end{lemma}


Lemma \ref{lemma:drehinvarianz} shows that we just need to evaluate $D_\setu(p)$ for polynomials $\tilde{p} \in \cB^{(k)}_m$ 
regardless of the transformation $Q \in V_k(\R^d)$ when taking a polynomial surrogate $p \in \mathrm{span}(\cB^{(d)}_m)$ of $f$.
Next, let us define $\cI_M(\bsaa) := \int_{\Omega^k} \bsx^\bsaa \, \mathrm{d}\mu(\bsx)$ with the restriction $\cI_M(\bsaa_\setu) := \int_{\Omega^{|\setu|}} 
\bsx_\setu^{\bsaa_{\setu}} \, \mathrm{d}\mu_{\setu}(\bsx_\setu)$ to the directions that are contained in $\setu$. Then, we have the following result.

\begin{lemma}  \label{lemma:DU_ana}
Let $p := \sum_{\substack{|\bsaa|_1 \leq m
}} C_{\bsaa} \bsx^{\bsaa} \in \cB^{(k)}_m$. Then, $\cA := D_\setu(p) + f_\emptyset^2$ from \eqref{eq:varViaSobol} fulfills
\begin{equation*}
	\cA = \sum_{\substack{|\bsaa|_1 \leq m \\ |\bsbb|_1 \leq m}} 
C_\bsaa C_\bsbb \ \cI_M(\bsaa_\setu + \bsbb_\setu) \cdot  \cI_M(\bsaa_{\setu^c})  \cdot  \cI_M(\bsbb_{\setu^c}) .
\end{equation*}
\begin{proof}
Using Lemma \ref{lemma:Du} we obtain
\begin{align*}
	&\cA = \int\limits_{\Omega^{2k-|\setu|}}  p(\bsx_\setu, \bsx_{\setu^c})\, p(\bsx_\setu, \bsy_{\setu^c}) \mathrm{d} \mu(\bsx) \mathrm{d} \mu_{\setu^c}(\bsy_{\setu^c}) \\
	&= \sum_{\substack{|\bsaa|_1 \leq m \\ |\bsbb|_1 \leq m }} C_\bsaa C_\bsbb \hspace{-0.3cm} \int\limits_{\Omega^{2k-|\setu|}}
	\hspace{-0.3cm} \bsx_{\setu}^{\bsaa_\setu + \bsbb_\setu} \, \bsx_{\setu^c}^{\bsaa_{\setu^c}} \, \bsy_{\setu^c}^{\bsbb_{\setu^c}} \mathrm{d} \mu(\bsx)
	\mathrm{d} \mu_{\setu^c}(\bsy_{\setu^c})
\end{align*}
and applying Fubini's theorem finishes the proof.
\end{proof}
\end{lemma}



In the case $\Omega = \R$ and when $\mu$ is the standard Gaussian measure 
we can compute $\cI_M(\bsaa)$ analytically.

\begin{lemma} \label{lemma:monomGauss}
The expected value of $\bsx^\bsaa$ with $|\bsaa|_1 = m$ for $\bsaa \in \N_0^k$ with respect to the Gaussian measure $\mu$ is
\begin{equation*} \label{eqn:monom_ew}
\E_\mu\left[\bsx^\bsaa \right] =
	  \begin{cases}
	  	\prod_{i=1}^k (\alpha_i-1)!! & \mbox{if all } \alpha_i \mbox{ even,}\\
	  	0 & \mbox{else,}
	  \end{cases}
\end{equation*}
where $n!! := n \cdot (n-2) \cdot (n-4) \cdot \ldots \cdot 1$. 
\begin{proof}
We have
$
  \E_\mu\left[\bsx^\bsaa \right] 
= \prod_{i=1}^k  \mathbb{E}_{\mu_i}(x^{\alpha_i})
$
and since the moments of the standard normal distribution are
\[
	\mathbb{E}(x^q) = \begin{cases}
	                  	(q-1)!! & q \mbox{ even}, \\
	                  	0 & q \mbox{ odd},
	                  \end{cases}
\]
the claim follows.
\end{proof}
\end{lemma}
With the help of Lemmata \ref{lemma:drehinvarianz}, \ref{lemma:DU_ana} and \ref{lemma:monomGauss},
we can compute $f_\emptyset$ and solve all the integrals involved in the computation of $\sigma^2_{\setu,\mu}(p \circ \bsQ)$ for all
$\setu \subseteq \cD$, $p \in \cP^{(d)}_m$ and $\bsQ \in V_k(\R^d)$
analytically if $\mu$ is taken as the standard Gaussian measure on the whole space $\R^k$. Therefore, we know how to evaluate $\hat{\MM}_p(\bsQ)$ 
in this case.


\subsection{A manifold CG algorithm on $V_k(\R^d)$.} \label{sec:CG}

To numerically solve the optimization problem, we propose a descent algorithm on the submanifold $V_k(\R^d) \subset \R^{d \times k}$.
Algorithm \ref{alg:manifoldCG} briefly sketches a 
conjugate gradient approach for this specific matrix manifold. Before we start the first iteration, we set the solution to the zero vector and 
choose a random direction.
For more details on the algorithm and the theory behind optimization on matrix manifolds, see 
\cite{AbsMahSep:2008}.

\begin{algorithm}
\caption{One step of the matrix manifold CG algorithm on $V_k(\R^d)$ to determine a maximizer of $\hat{\MM}_p$.}
\begin{algorithmic}
 \STATE \textbf{Input:} Iterate $\bsQ \in V_k(\R^d)$, direction $\bsM \in \R^{d \times k}$.
 \STATE \textbf{Output:} New iterate $\bar{\bsQ} \in V_k(\R^d)$, new direction $\bar{\bsM} \in \R^{d \times k}$.
 \vspace{0.15cm}
 \STATE Line search for $\delta > 0$ along $\bsQ + \delta \bsM$.
 \STATE QR decomposition: $QR \gets \bsQ + \delta \bsM$ (``\emph{retraction}'').
 \STATE Define new iterate: $\bar{\bsQ} \gets Q$.
\STATE Polak-Ribiere: $\beta \gets \frac{\nabla \hat{\MM}_p(\bar{\bsQ})^T \left( \nabla \hat{\MM}_p(\bar{\bsQ}) - \nabla \hat{\MM}_p(\bsQ) \right)}{ \nabla \hat{\MM}_p(\bsQ)^T  \nabla \hat{\MM}_p(\bsQ)}$.
 \STATE $\bsM^* \gets (\bsI - \bar{\bsQ}\bar{\bsQ}^T)\bsM + \frac{1}{2}\bar{\bsQ}(\bar{\bsQ}^T\bsM - \bsM^T\bar{\bsQ})$ (``\emph{parallel transport}'').
 \STATE Define new direction: $\bar{\bsM} \gets - \nabla \hat{\MM}_p(\bar{\bsQ}) + \beta \bsM^*$.
\end{algorithmic}
 \label{alg:manifoldCG}
\end{algorithm}

As we see, we need matrix-matrix multiplications, a QR decomposition of a $d \times k$ matrix and 
a gradient computation of $\hat{\MM}_p$, which has to be understood as calculating 
the usual gradient in $\R^{d \times k}$. We do the latter by using first order forward finite differences.
\begin{proposition} \label{prop:costs}
 Let $\mu = \cN(\mathbf{0},\bsI)$ be the standard Gaussian measure. The total number of operations to perform one CG iteration of Algorithm \ref{alg:manifoldCG} is bounded by
 \begin{align*} & \ {\cal O}\left(kd \cdot \left(d \binom{d+m}{d} + \left( k \binom{k+m}{k}
 \right)^2 \right) \right) \\ = & \ {\cal O}\left(k d^{m+2} + k^{2m+3} d \right).
\end{align*}
 \begin{proof}
Since computing the QR decomposition of a $d \times k$ matrix costs ${\cal O}\left(k^2 \cdot d \right)$ operations and the costs for computing the matrix products in the parallel transport step scale like
${\cal O}\left( k^2 \cdot d + k \cdot d^2\right) = {\cal O}\left( k \cdot d^2 \right)$, we directly see that the most expensive part
is computing the derivative
$\nabla \hat{\MM}_p(\bar{\bsQ})$ by finite differences, for which we need to perform $2kd$ evaluations of $\hat{\MM}_p$. 
To this end, let $\bsQ \in V_k(\R^d)$ be given and assume that we want to compute  $\hat{\MM}_p(\bsQ)$. We first need to calculate the coefficients of $p \circ \bsQ$ in the basis $\cB^{(k)}_m$
as done in the proof of Lemma \ref{lemma:drehinvarianz}.
Here, we store the intermediate values $s_i := \sum_{j=1}^k Q_{ij} x_j$ for all $i=1,\ldots,d$ and subsequently compute
$
\sum_{|\bsaa|_1 \leq m} c_{\bsaa} \prod_{i=1}^d s_i^{\alpha_i}
$
for the coefficients $c_{\bsaa}$ of $p$. This costs ${\cal O}(d \cdot \binom{d+m}{d})$ operations since there are $\binom{d + m}{d}$ monomials. 
Next, we evaluate $\hat{\MM}_p$ at $\bsQ$. To this end, note that the computation of $D_\setu(p \circ \bsQ)$ - with given coefficients for $p \circ \bsQ$ 
- takes ${\cal O}(\binom{k + m}{k} \cdot \binom{k + m}{k} \cdot k)$ operations as proven in Lemma \ref{lemma:DU_ana} and Lemma \ref{lemma:monomGauss}. 
It remains
to show that we can evaluate $\hat{\MM}_p(\bsQ)$ by using only ${\cal O}(k)$ different sets $\setu \subset \{ 1,\ldots,k\}$ and their corresponding $D_\setu(p \circ \bsQ)$. 
Indeed, let $\hat{\nu}_{\setu} := \exp\left(- \max \{ j \in \setu \} \right)$ and let $[i] := \{1,\ldots,i\}$. According to \eqref{eq:finalMaxProb}, we actually need to compute 
\begin{align*} 
 \hat{\MM}_p(\bsQ) & = \sum_{\setu \subseteq [k]} \hat{\nu}_{\setu} \sigma^2_{\setu, \mu}(f \circ \psi) \\
 & = \sum_{\setu \subseteq [k-1]}  \hat{\nu}_{\setu} \sigma^2_{\setu, \mu}(f \circ \psi) + \exp\left(-k\right) \cdot \left(D_{[k]}(p \circ \bsQ) -  D_{[k-1]}(p \circ \bsQ)  \right) \\
 & = \ldots \\
 & = \sum_{i=1}^k \exp\left(-i\right) \cdot \left(D_{[i]}(p \circ \bsQ) -  D_{[i-1]}(p \circ \bsQ) \right),
\end{align*}
where we set $D_{\emptyset} := 0$. This can be done with ${\cal O}(k)$ evaluations of different $D_{\setu}$, which completes the proof.
\end{proof}
\end{proposition}

For small $m$, algorithm \ref{alg:manifoldCG} is applicable in moderate dimensions ($d \lesssim 50$), which could be improved by using smaller polynomial spaces, e.g.\ based on hyperbolic crosses.
Finally note that a similar approach with a Newton-type 
optimizer has been introduced in \cite{ConstantineHokanson}. 
There 
the surrogate $p$ for $f$ and the transformation $\bsQ \in V_k(\R^d)$ are optimized at the same time.
In our case, the polynomial just serves to coarsely represent 
the ANOVA structure of $f$. Since we want to obtain a cost-efficient representation of $f \circ \bsQ \approx p \circ \bsQ$ before actually solving the underlying transformed
least-squares problem on a sparse grid, we take the generalized mean dimension $\hat{\MM}_f$ as a complexity indicator to optimize with respect to $\bsQ$. 
Subsequently, we solve a least squares problem in a sparse grid space as described in the next section.

\section{Application to regression with sparse grids}

In this section, we briefly recapitulate least squares regression on sparse grids, see e.g.\ \cite{Garcke.Griebel.Thess:2000, pfluegerpeher}, and 
discuss a space- and dimension-adaptive variant based on the concepts of \cite{Feuersaenger:2010, pfluegerpeher}. 
Then, to reduce the computational costs of the adaptive regression algorithm, we suggest a preprocessing step based on the preceding section.

\subsection{Multivariate regression on sparse grids}
Let $\bsz := {\{ (\bst_i, x_i) \in T \times \R \mid i=1,\ldots,N \}}$ be $N$ given samples, where $T := [0,1]^d$. If a different domain is used, the data has to be rescaled appropriately.
Generally, least squares regression determines a minimizer to
\begin{equation} \label{sampleE}
 \min_{f \in S} {\cE}_{\bsz} ~~ \text{ with } ~~ {\cE}_{\bsz}(f) := \frac{1}{N} \sum_{i=1}^N (x_i - f(\bst_i))^2
\end{equation}
over some set of functions $S$.
If $S$ is a reproducing kernel Hilbert space the famous representer theorem states that the solution can be determined by solving a (usually dense) $N \times N$ system of 
linear equations \cite{Smola}. Naively, this would need ${\cO}(N^3)$ floating point operations.
While some algorithms solve these kernel systems approximately, they still scale worse than linearly in $N$ in general.
Therefore, we will employ a space $S$ of sparse grid functions instead, which naturally leads to an algorithm that scales linearly in $N$ and circumvents the curse of dimensionality 
of standard tensor grid approaches \cite{Bungartz.Griebel:2004}. Several variants of such sparse grid least squares algorithms 
have been very successfully employed for different regression tasks, see e.g.\ \cite{Garcke.Griebel.Thess:2000, pfluegerpeher}. 

We shortly recall the sparse grid discretization based on the modified linear basis from \cite{pfluegerpeher}.
Let
\begin{equation*}
  \Phi(t) := \max\left(1 - |t|,0\right) ~ \text{ and } ~ \Phi_{l,i}(t) := \Phi(2^l \cdot t - i)|_{[0,1]}
\end{equation*}
for $l,i \in \N_+$.
To construct the modified hierarchical linear basis let $I_l :=\{ i \in \mathbb{N}_+ \mid 1 \leq i \leq 2^{l} - 1, ~ i \text{ odd} \}$. 
For $l = 1$ we set $\gamma_{1,1} := 1$. For $l \geq 2$ we define
$ \gamma_{l,i} := \Phi_{l,i}$ for $i \in I_l \setminus \{1, 2^{l}-1\}$ and
$
\gamma_{l,1}(t) := \max\left(2 - 2^l t,0\right)|_{[0,1]}$ and $\gamma_{l,2^l -1}(t) := \gamma_{l,1}(1-t)$, see also Figure \ref{fig:sparseGrids}(left).

\begin{figure}[tb]
\centering
\includegraphics[height=4.5cm]{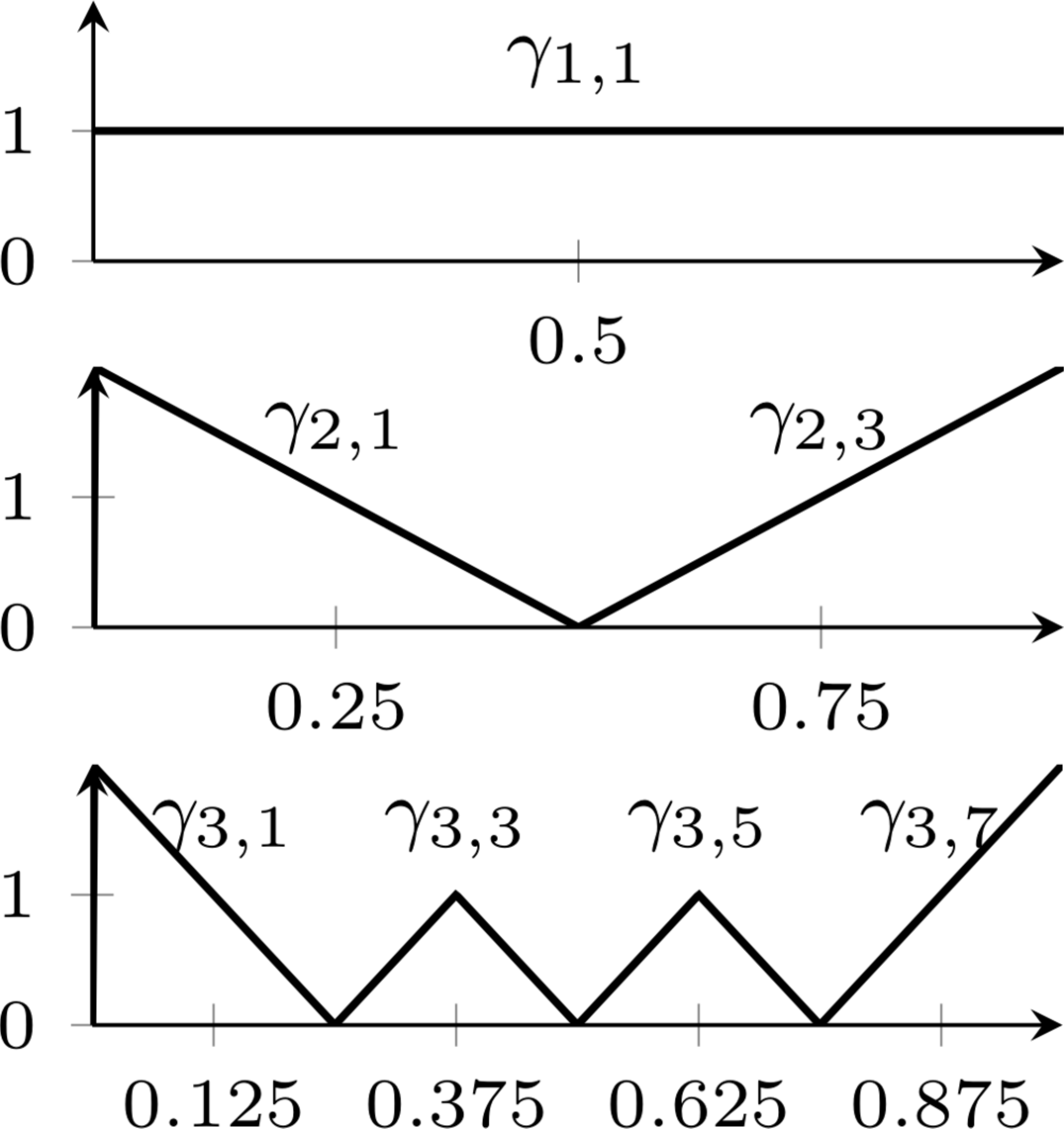} \hspace{0.5cm} 
 \includegraphics[height=4.4cm]{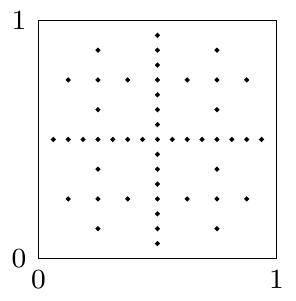} 
 \caption{Hierarchical basis functions up to $l = 3$ (left) and regular $2d$ sparse grid of level $\ell = 4$ (right).}
 \label{fig:sparseGrids}
\end{figure}

 The $d$-variate basis functions are then built via the tensor product construction
\begin{equation*}
 \gamma_{\mathbf{l,i}}(\bst) := \prod_{j=1}^d \gamma_{l_j,i_j}(t_j),
\end{equation*}
where $\mathbf{l} = (l_1,\ldots,l_d) \in \mathbb{N}_+^d$ is the multivariate level
 and $\mathbf{i} = (i_1,\ldots,i_d) \in \mathbb{N}_+^d$ denotes the multivariate position index.
 Let $\mathbf{I_l} := \otimes_{j=1}^d I_{l_j}$.
Then, $W_{\mathbf{l}} := \text{span} \left\{ \gamma_{\mathbf{l,i}} \mid \mathbf{i} \in \mathbf{I_l} \right\}$ denotes the so-called 
hierarchical increment space of level $\mathbf{l}$. 
 We now define the (regular) sparse grid space of level $\ell > 0$ by 
 \begin{equation} \label{sgspace}
 V^\ell := \bigoplus_{\mathbf{k} \in \mathbb{N}_+^d, |\mathbf{k}|_1 \leq \, \ell + d - 1} W_{\mathbf{k}}.
\end{equation}
 Instead of $2^{\ell d}$ degrees of freedom as in the full grid case, the sparse grid space only contains $M := \dim(V^l) = \cO(2^{\ell} \ell^{d-1})$ basis functions. 
 A $2d$ sparse grid, i.e.\ the centers of the supports of all basis function of $V^\ell$, can be found in Figure \ref{fig:sparseGrids}(right).
 For more details on sparse grids and a thorough comparison to full grids regarding cost complexity and approximation rates we refer to \cite{Bungartz.Griebel:2004}.
 
 Representing $f \in V^\ell$ in the hierarchical basis yields
 $$
 f(\bst) = \sum_{| \mathbf{k} |_1 \leq \ell + d - 1} \sum_{\mathbf{i} \in \mathbf{I_k}} \beta_{\mathbf{k,i}} \gamma_{\mathbf{k,i}}(\bst).
 $$
We now consider the Tikhonov-regularized version
 \begin{equation} \label{discTikReg} 
 \min_{f \in S}  \frac{1}{n} \sum_{i=1}^n (x_i - f(\bst_i))^2 + \lambda \| \vec{\bsbb} \|^2_{2}
 \end{equation}
of the least squares problem \eqref{sampleE} for $S = V^{\ell}$. Then, the coefficient vector $\vec{\bsbb}$ is given by
 \begin{equation} \label{eq:LGS} 
\left( {\bsB}^T {\bsB} + \lambda \bsI \right) \vec{\bsbb} = {\bsB}^T \vec{\bsx},
 \end{equation}
  where $\bsB \in \R^{N\times M}$ with entries $\bsB_{i,(\mathbf{l,j})} = \gamma_{\mathbf{l,j}}(\bst_i)$ 
and $\vec{\bsx} = (x_1,\ldots,x_N)^T$.
%
 We employ a conjugate gradient solver to obtain the solution. 
 For details on the linear equation system and the fast numerical treatment in the sparse grid case we refer to \cite{Feuersaenger:2010,pfluegerpeher}.
 
\subsection{Adaptive sparse grids.}

If certain spatial directions or regions are more important than others, 
e.g.\ when the solution of \eqref{discTikReg} varies strongly 
in one part of the domain but is almost constant in others, it is reasonable to adjust the underlying discretization to this behavior. 
To this end, space- and dimension-adaptive sparse grids can be employed \cite{Feuersaenger:2010,pfluegerpeher}. 
They adapt according to an error indicator which determines 
where the grid will be refined.
Here, we use a combination 
$$
\epsilon_{\mathbf{l,i}} := \beta_{\mathbf{l,i}} \sum_{j=1}^N \gamma_{\mathbf{l,i}}(\bst_j) \cdot \left(f(\bst_j) - x_j \right)^2,
$$
between the coefficients and the least-squares error. This serves to 
indicate how much $\gamma_{\mathbf{l,i}}$ contributes to the least squares error in the actual discretization.

The actual adaptive algorithm starts with $V = V^\ell$ for small $\ell$ 
and solves \eqref{discTikReg} to obtain the solution $f \in V$. 
Then, an initial compression step is performed, i.e.\ we mark all those basis functions from 
$V$ for which $\epsilon_{\mathbf{l,i}}$ is smaller than a fixed threshold. Subsequently, all marked basis functions are removed from $V$. 
However, due to the hierarchical structure, we do not remove $\gamma_{\mathbf{l,i}}$ if one of its successors, i.e.\ a
basis function whose support is a subset of the support of $\gamma_{\mathbf{l,i}}$, is not marked. 
Finally, we run a series of refinement steps, which consist of solving \eqref{discTikReg} over $V$, marking the $\cL > 0$ 
refinable\footnote{We call a basis function refinable 
if not all of its children are already included in the grid. By children of $\gamma_{\mathbf{l,i}}$ 
we mean all successors on levels $\mathbf{l} + \mathbf{e}_k$, where $\mathbf{e}_k$ denotes the $k$-th unit vector.}
basis functions in $V$ with the largest value of $\epsilon_{\mathbf{l,i}}$ and then refining the marked 
functions.
In this paper, we consider two different kinds of refinement: 
The first one, referred to as ``standard'' refinement, inserts all $2d$ children of each marked function. 
The second one, referred to as ``ANOVA'' refinement, only inserts children
in those directions $k$, for which $l_k > 1$. This ensures that $f$ remains constant in directions which the compression step has deemed to be 
irrelevant. The complete space- and dimension-adaptive procedure is described in Algorithm \ref{alg:dimadp}.
For more details on adaptivity, its relation to the anchored ANOVA decomposition and fast sparse 
grid traversal algorithms we refer to \cite{Feuersaenger:2010}.

\begin{algorithm} 
\caption{The adaptive sparse grid algorithm}
\begin{algorithmic}
 \STATE \textbf{Input:} $l \in \N$, Threshold $t > 0$, $\cL \in \N$, numIt $\in \N$. 
 \STATE \textbf{Output:} Adaptive sparse grid space $V$.
 \vspace{0.15cm}
\STATE Solve \eqref{discTikReg} over $V := V^l$ and compress($V$, $t$).
\FOR{$i=l\ldots\,$numIt}
 \STATE Solve \eqref{discTikReg} over $V$ and refine($V$, $\cL$).
\ENDFOR
\end{algorithmic}
\label{alg:dimadp}
\end{algorithm}

\subsection{The final preprocessing method.}


Let $\bsaa_1 := \mathbf{0}$ and let $\bsaa_i, i=2,\ldots, \bar{K} = \binom{d + m}{d}$ be an arbitrary enumeration of all 
indices $|\bsaa|_1 \leq m$ corresponding to the basis set $\cB^{(d)}_m$ of polynomials with total degree less than $m$.
The final ingredient to our overall algorithm is solving the unregularized least squares problem \eqref{sampleE} over
$\mathrm{span}(\cB^{(d)}_m)$ to build the polynomial surrogate $p = \sum_{i = 1}^{\bar{K}} c_{\bsaa_i} \bst^{\bsaa_i}$. 
 Its coefficients are the minimizers of
 \begin{equation} \label{eqn_quad_regress}
\min_{\vec{\bsw} \in \R^{\bar{K}}} \| \bsA \vec\bsw - \vec\bsx \|_2 = \min_{\vec{\bsw} \in \R^{\bar{K}}} \| \bsW \vec\bsw - \bsV^T\vec\bsx \|_2
\end{equation}
 and can be determined by backsubstitution after computing a QR decomposition
 of the Vandermonde matrix
 \begin{equation*} \label{eqn:LSmatrix}
	\bsV \bsW = \bsA := \left( \begin {array} {c c c c c} 
        1 & \bst_{1}^{\bsaa_2} & \ldots & \bst_{1}^{\bsaa_{\bar{K}}}  \\
	\vdots & \vdots & & \vdots \\
	1  & \bst_{N}^{\bsaa_2} & \ldots & \bst_{N}^{\bsaa_{\bar{K}}} \\ 
	\end {array} \right).
\end{equation*}

Now, let $\cC$ be the cumulative distribution function of the $k$-variate normal distribution $\cN(\mathbf{0}, \mathbf{I})$, 
which we apply to rescale the data to $[0,1]^k$ for the sparse grid algorithm. 
The final optimally rotated, adaptive sparse grid least-squares method is presented in Algorithm \ref{alg:final}.
If the distribution of the input data is known to be non-Gaussian, it might 
be more sensible to use a different transformation for the rescaling 
onto $[0,1]^k$.

\begin{algorithm} 
\caption{Optimally rotated, adaptive sparse grid least-squares algorithm}
\begin{algorithmic}
 \STATE \textbf{Input:} Initial data $\bsz = \{(\bst_i,x_i) \mid i = 1,\ldots,N \}$.
 \STATE \textbf{Output:} A $\bsQ \in V_k(\R^d)$ and a sparse grid function $f: [0,1]^k \to \R$ such that 
  $ f \circ \cC \circ \bsQ^T $ approximates $\bsz$.
 \vspace{0.15cm}
 \STATE Determine $p$ via \eqref{eqn_quad_regress}.
 \STATE Determine $\bsQ$ with Algorithm \ref{alg:manifoldCG}. 
 \STATE Transform data to $\tilde{\bsz} = (\cC(\bsQ^T \bst_i), x_i)_{i=1}^N$.
 \STATE Compute $f: [0,1]^k \to \R$ with Algorithm \ref{alg:dimadp} on $\tilde{\bsz}$. 
\end{algorithmic}
\label{alg:final}
\end{algorithm}


\section{Numerical results}
For our computations, we employ the SG++ sparse grid library \cite{pfluegerpeher} and choose the following para\-meters for all experiments: total degree $m = 3$, truncation parameter $k = \min(d,3)$, 
 compression threshold $t = 0.1$, number of points to refine $\cL = 10$, initial grid level $l = 3$. 
The CG algorithm for solving \eqref{eq:LGS} is iterated 
until the norm of the residual has decreased by a factor of $10^{-12}$.
To measure our performance, we use the normalized RMSE
$$
\text{NRMSE} :=  \sqrt{\frac{\sum_i \left(f\left(\cC\left(\bsQ^T(\tilde{\bst}_i)\right)\right) - \tilde{x}_i \right)^2}{\sum_i \tilde{x}_i^2}},
$$
where $(\tilde{\bst}_i,\tilde{x}_i)$ is some test data.
We compare this value to the NRMSE of Algorithm \ref{alg:dimadp} on 
the untransformed data.
Note that the runtimes of Algorithm \ref{alg:final} were (often magnitudes) smaller than the runtimes 
of Algorithm \ref{alg:dimadp} on the untransformed data set for all of our experiments.

\subsection{Two-dimensional ridge function.}
We draw $N = 10^5$ i.i.d.\ $\cN(\mathbf{0},\mathbf{I})$ distributed points
$\bst_i \in \R^2$ and choose the ridge function
$x_i = \tanh([\bst_i]_1 + [\bst_i]_2) + \varepsilon_i$, i.e.\ we evaluate $\tanh$ on the sum of the coordinates of each data vector 
and add i.i.d.\ white noise $\varepsilon_i \sim \cN(0,10^{-8})$.
We also create a test data set $(\tilde{\bst}_i, \tilde{x}_i)$ of size $N$ with the same distribution but
without the noise. 
Since $N$ is significantly larger than the sparse grid sizes we use, we set $\lambda = 0$. 
The refinement process is iterated until the number of grid points exceeds $500$.
The resulting errors for each refinement after the initial compression and the employed sparse grids are illustrated 
in Figure \ref{fig:TanhsparseGrids} for both Algorithm \ref{alg:final} and Algorithm \ref{alg:dimadp} on original data.

 As we observe, Algorithm \ref{alg:final} achieves an NRMSE, which is several magnitudes 
 smaller than the NRMSE of the adaptive sparse grid algorithm on the untransformed data for both ANOVA and standard refinement. 
 Obviously, the ANOVA refinement is too restrictive in the untransformed case and only standard refinement seems to converge.
 The remarkable performance of Algorithm \ref{alg:final} is obvious as the specific ridge function example has a rotated one-dimensional structure,  
 which the preprocessing is able to pick up. 
This can be seen in the grid in  
 Figure \ref{fig:TanhsparseGrids}(right), where most points are spent along the horizontal line in the middle of the domain. 

 \begin{figure}[t]
 \centering
  \includegraphics[width=8.5cm]{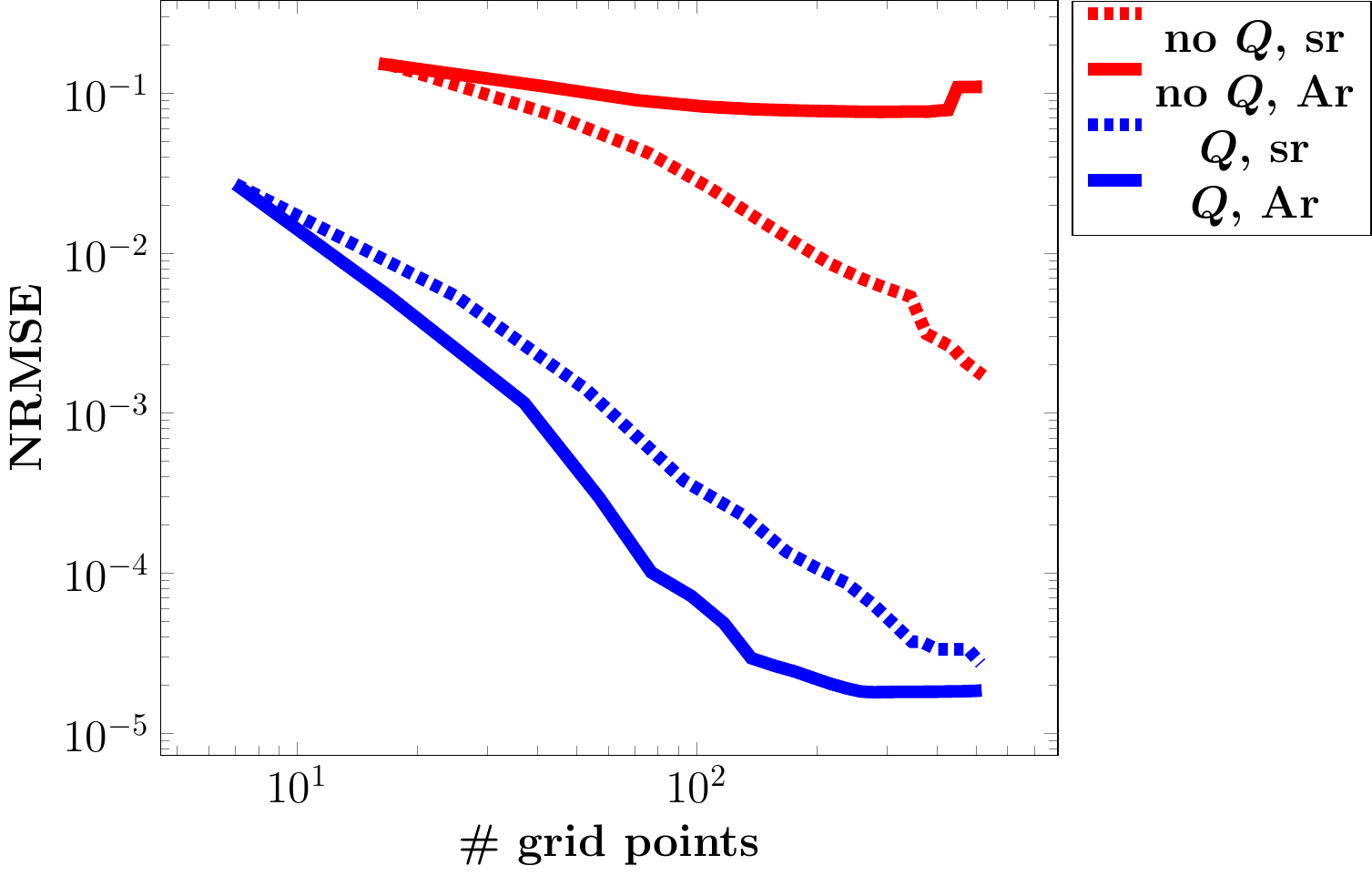}
 \includegraphics[height=4.5cm]{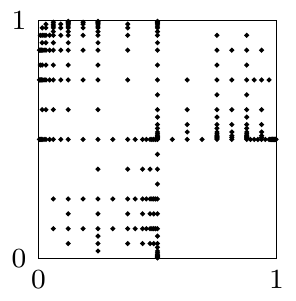} \hspace{0.5cm}
 \includegraphics[height=4.5cm]{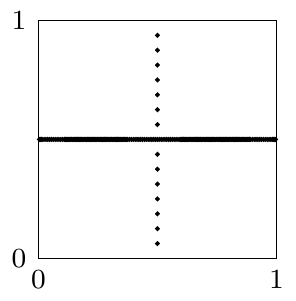}
 \caption{NRMSE for $2$d $\tanh$ ridge function (top). Ar = ANOVA refinement, \mbox{sr = standard refinement}.
 First $200$ sparse grid points inserted by Ar for untransformed data (bottom left) 
 and transformed data (bottom right).}
 \label{fig:TanhsparseGrids}
\end{figure}

\subsection{Five-dimensional sum of ridge functions.}
We draw $N=10^5$ points $\bst_i \sim \cN(\mathbf{0},\mathbf{I})$ in $\R^5$ and use 
$x_i = \tanh(\sum_{j=1}^5 [\bst_i]_j) + \max(0, \sum_{j=1}^5 (-1)^j [\bst_i]_j) + \varepsilon_i$ with $\varepsilon_i \sim \cN(0,10^{-8})$. 
Since this test function is non-smooth, it is more complicated than the one from the last section. Nonetheless, it is a simple sum of two ridge functions 
and our algorithm 
should be able to exploit this. We set $\lambda = 0$ and terminate the adaptive algorithm after the number of grid points has reached $1000$. 
The results can be found in Figure \ref{fig:5dridge}.
As in the previous example, we clearly see that the data transformation benefits the adaptive sparse grid algorithm significantly.

\begin{figure}[ht]
 \centering
 \includegraphics[width=8.5cm]{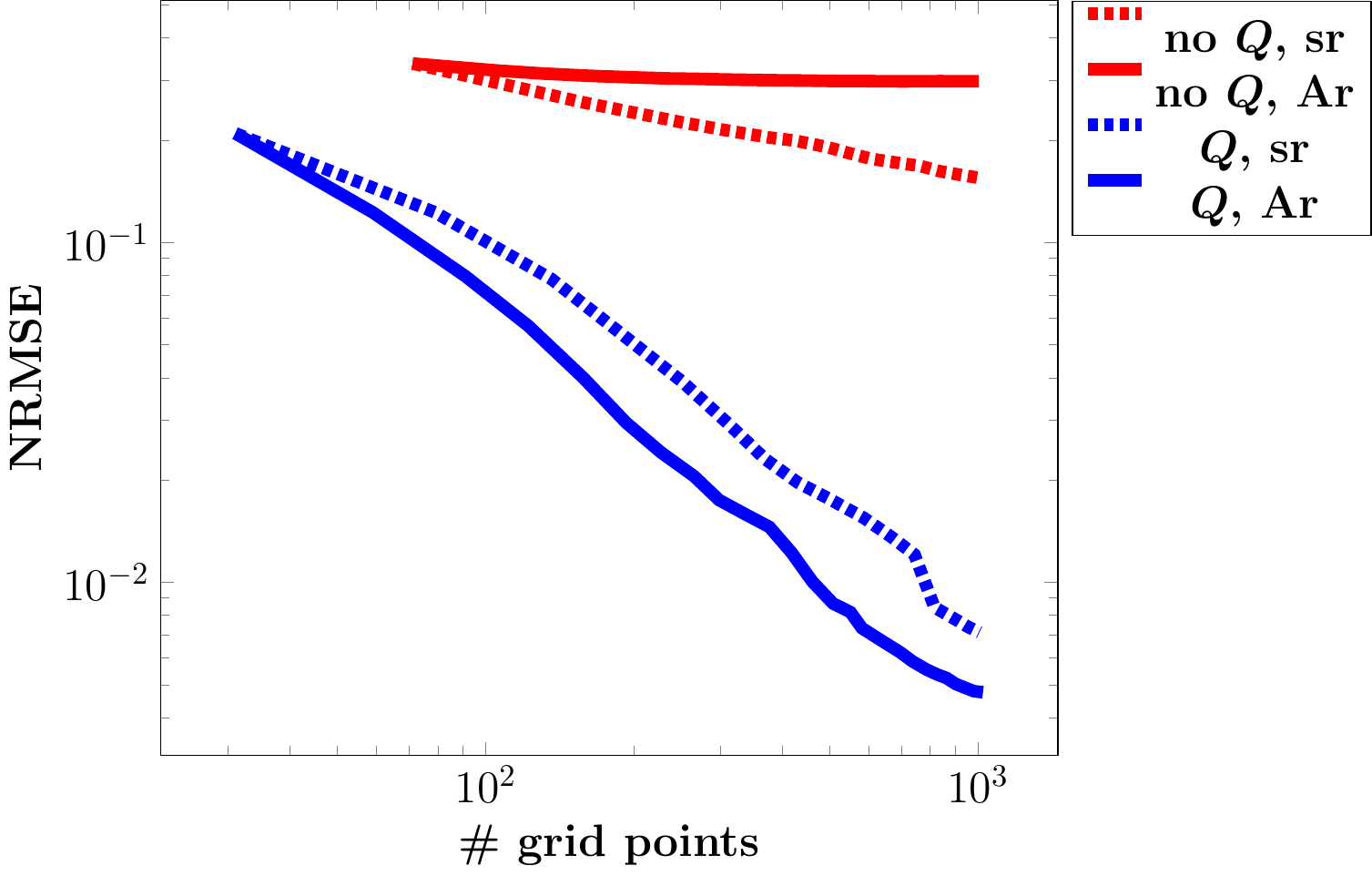}
 \caption{NRMSE for $5$d sum of ridge functions. Ar = ANOVA refinement, \mbox{sr = standard refinement.}}
 \label{fig:5dridge}
\end{figure}

\subsection{Ten-dimensional PDE problem.}
In this example from \cite{ConstantineHokanson}, $10^4$ vectors $\bst_i \in \R^{10}$ are drawn according to $\cN(\mathbf{0},\mathbf{I})$. 
These reflect the parameters of the diffusion coefficient $a$ in the two-dimensional elliptic PDE 
$$
- \nabla_\bss \cdot \left(a(\bss,\bst) \nabla_\bss u(\bss,\bst)  \right) = 1 ~~~~~~~~ \bss \in [0,1]^2
$$
with Neumann boundary conditions on the right side of the domain and Dirichlet zero boundary conditions on the other sides.
The $\bst_i$ represent the first ten coefficients of a truncated Karhunen-Lo\'eve decomposition of $\log(a)$ 
with correlation kernel $\exp(- \| \bsr - \bss \|_1 )$. For each $\bst_i$, the PDE is solved and $x_i$ is set to 
the spatial average of the solution on the Neumann boundary. 
Solving regression problems of this kind is an important task in uncertainty quantification, see \cite{ConstantineBook} 
for details.
We present averaged results over $20$ random splits of the data into $5000$ training and $5000$ test points for different regularization parameters $\lambda \in \{10^{-2}, 10^{-4}, 10^{-6}\}$
in Figure \ref{fig:PDE}.

The smallest error with the least amount of grid points is achieved with transformed data and ANOVA refinement. For this example, the ANOVA refinement performs
better than standard refinement also in the case of untransformed data. 
However, standard refinement seems to produce more stable results with respect to $\lambda$. 
For ANOVA refinement, transformed data and $\lambda = 10^{-4}$ we achieve an averaged NRMSE smaller than $0.1$, which is competitive with the best results from \cite{ConstantineHokanson}. 
For all choices of $\lambda$, we also outperform the LASSO and Gaussian processes approaches tested there.
\begin{figure}
 \centering
 \includegraphics[width=10cm]{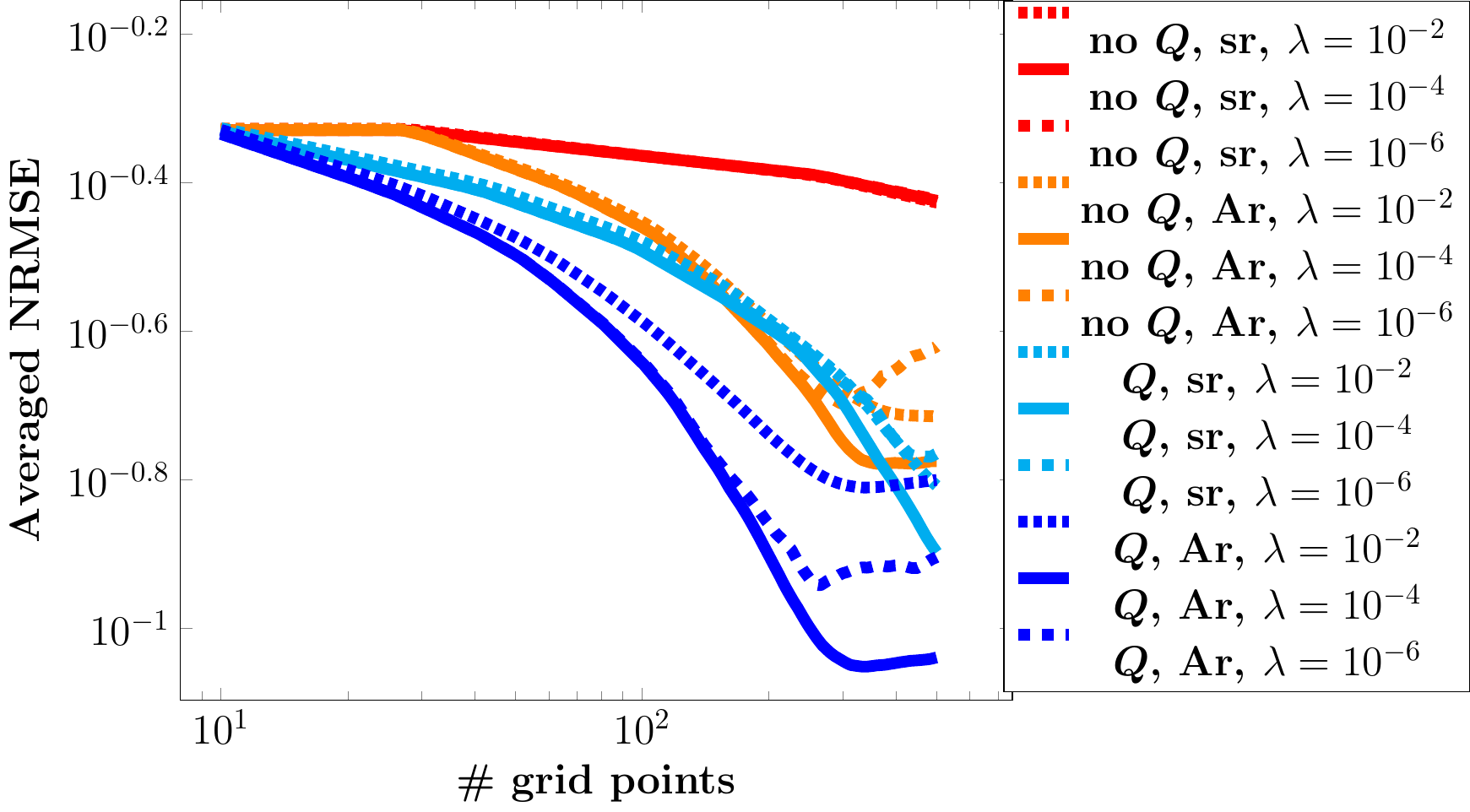}
 \caption{Averaged NRMSE for the $10$d PDE problem. Ar = ANOVA refinement, \mbox{sr = standard refinement}.}
 \label{fig:PDE}
\end{figure}

\section{Conclusion}

In this paper we have discussed the idea of preprocessing data in regression tasks in order to achieve a beneficial error decay and possibly smaller computational costs
of the underlying algorithm. Our approach is motivated by the ANOVA decomposition and works best with regression methods based 
on tensor-product functions, 
e.g.\ on full grids and sparse grids. 
We provided an efficient algorithm 
to find the optimal matrix $\bsQ \in V_k(\R^d)$ to transform the data at hand. 
Subsequently, we 
discussed an adaptive sparse grid least-squares regression algorithm, which is able to adapt to the 
underlying regressor function.
We showed how our preprocessing method significantly enhances the performance of the adaptive sparse grid algorithm for 
both artificial toy problems and a real-world application from uncertainty quantification. 


\bibliography{literature}
\bibliographystyle{abbrv}

\end{document}